\documentclass[letter,11pt]{article}
 \makeatletter
\def\@fnsymbol#1{\ensuremath{\ifcase#1\or \dagger\or \ddagger\or
   \mathsection\or \mathparagraph\or \|\or **\or \dagger\dagger
   \or \ddagger\ddagger \else\@ctrerr\fi}}
\makeatother

\usepackage[english]{babel}
\usepackage[utf8]{inputenc} 
\usepackage[T1]{fontenc}    
\usepackage[colorlinks,allcolors=blue]{hyperref}

\usepackage{lmodern}
\usepackage{url}            
\usepackage{booktabs}       
\usepackage{amsfonts}       
\usepackage{nicefrac}       
\usepackage{microtype}      
\usepackage{enumitem}
\usepackage{dsfont}
\usepackage{subfig}
\usepackage{wrapfig, lipsum}
\usepackage{graphicx}
\usepackage[titletoc,title]{appendix}
\usepackage{amsmath,amsthm,amssymb,mathtools}
\usepackage{algorithm}
\usepackage{float}
\usepackage{sidecap}

\usepackage{braket}

\newtheorem{theorem}{Theorem}[section]
\newtheorem*{theorem*}{Theorem}

\newtheorem*{proposition*}{Proposition}
\newtheorem{lemma}[theorem]{Lemma}
\newtheorem*{lemma*}{Lemma}

\newtheorem*{conjecture*}{Conjecture}

\newtheorem*{fact*}{Fact}

\newtheorem*{hypothesis*}{Hypothesis}

\newtheorem*{claim*}{Claim}

\theoremstyle{definition}

\newtheorem{assumption}[theorem]{Assumption}

\theoremstyle{remark}

\newtheorem*{remark*}{Remark}

\newtheorem*{observation*}{Observation}

\renewcommand{\epsilon}{\varepsilon}

\newif\ifnotes\notesfalse

\ifnotes
\usepackage{color}
\definecolor{mygrey}{gray}{0.50}
\newcommand{\notename}[2]{{\textcolor{mygrey}{\footnotesize{\bf (#1:} {#2}{\bf ) }}}}

\else

\newcommand{\notename}[2]{{}}

\fi

\usepackage{tikz}
\usepackage{multirow}
\usepackage[noend]{algpseudocode}

\usepackage{amsmath}
\usepackage{graphicx}
\usepackage[colorinlistoftodos]{todonotes}

\usepackage{fullpage}

\newcommand{\inprod}[1]{\langle #1 \rangle}
\newcommand{\opt}{^\star}
\newcommand{\kp}{_{k+1}}

\newcommand{\mQ}{\mathcal{Q}}
\newcommand{\mF}{\mathcal{F}}
\newcommand{\mS}{\mathcal{S}}
\newcommand{\mD}{\mathcal{D}}
\newcommand{\mR}{\mathcal{R}}
\newcommand{\mL}{\mathcal{L}}
\newcommand{\mN}{\mathcal{N}}

\newcommand{\bbE}{\mathbb{E}}

\usepackage{amsmath,amsfonts,bm}

\def\eqref#1{equation~\ref{#1}}

\def\1{\bm{1}}

\DeclareMathAlphabet{\mathsfit}{\encodingdefault}{\sfdefault}{m}{sl}
\SetMathAlphabet{\mathsfit}{bold}{\encodingdefault}{\sfdefault}{bx}{n}

\def\sR{{\mathbb{R}}}

\DeclareMathOperator{\sign}{sign}

\title{Towards Accurate Quantization and Pruning via Data-free Knowledge Transfer}
\author{Chen Zhu\thanks{University of Maryland, College Park. {\tt \{chenzhu,xuzh,ashafahi,manlis,ghiasi,tomg\}@umd.edu}} \and
 Zheng Xu\footnotemark[1] \and
 Ali Shafahi\footnotemark[1] \and
 Manli Shu\footnotemark[1] \and
 Amin Ghiasi\footnotemark[1] \and
 Tom Goldstein\footnotemark[1] 
}
\date{}

\begin{document}
\maketitle

\begin{abstract}
When large scale training data is available, one can obtain compact and accurate networks to be deployed in resource-constrained environments effectively through quantization and pruning. However, training data are often protected due to privacy concerns and it is challenging to obtain compact networks without data. We study data-free quantization and pruning by transferring knowledge from trained large networks to compact networks. Auxiliary generators are simultaneously and adversarially trained with the targeted compact networks to generate synthetic inputs that maximize the discrepancy between the given large network and its quantized or pruned version. We show theoretically that the alternating optimization for the underlying minimax problem converges under mild conditions for pruning and quantization. Our data-free compact networks achieve competitive accuracy to networks trained and fine-tuned \textit{with} training data. 
Our quantized and pruned networks achieve good performance while being more compact and lightweight. 
Further, we demonstrate that the compact structure and corresponding initialization from the Lottery Ticket Hypothesis can also help in data-free training. 

\end{abstract}

\section{Introduction}
Deep neural networks (DNNs) have been applied to a wide range of tasks and applications in computer vision and sequence modeling. DNNs with impressive performance are often huge models with a large number of parameters and high computational cost, which limits their deployment on resource constrained devices with limited memory and processing power. 
With the emergence of edge devices and wide-ranging applications of deep neural networks, the demand for lightweight neural networks has increased. 

To address this demand, many methods have been proposed to obtain lightweight models with modest computational/memory costs without a great sacrifice in performance compared to the full model. Some common techniques include knowledge distillation, quantization, and network pruning.
Knowledge distillation works by enforcing the smaller network named the student to generate outputs similar to those of the trained larger network named the teacher \cite{hinton2015distilling}. 
Quantization refers to reducing the number of bits for representing network parameters or their activations \cite{courbariaux2015binaryconnect}. 
Network pruning corresponds to keeping a minimal set of network parameters \cite{han2015deep}. All these methods, in their conventional setting, require  some kind of access to the training set to achieve their best performance. 
While the availability of training data is a viable assumption for public datasets, there exists many critical cases where the training data is inaccessible due to concerns about protecting privacy of the users or the intellectual properties of the corporations~\cite{taigman2014deepface,wu2016google,micaelli2019zero}.
These practical limitations motivate us to seek solutions for compressing deep models without accessing training data. 

\subsection*{Contributions}
Given a pre-trained large scale model with high performance on practical applications, we study \emph{data-free} methods for training compact models that can run on resource-limited devices. Our contributions are:
\begin{itemize}
    \item We train compact networks with fast inference capacity and low memory footprint by combining knowledge distillation, quantization, and pruning under an adversarial training framework, where an auxiliary network is adversarially trained to find the worst case synthetic data that differentiates between the given larger network and the target compact network. 
    \item Our method can quantize networks to use extreme low-bit, i.e. binary, representations for weights without noticeable performance degradation, which was not possible by previous data-free methods. 
    \item We compress large networks by pruning the weights of the original network to compression ratios previously only possible with fine-tuning on a large number training data.  
    \item We analyze the convergence of the alternating optimization used for solving the minimax problem of our proposed method. For quantization, we prove an $O(1/\sqrt{k})$ convergence rate for the error bound of convex-concave objectives and bounded gradient variance assumptions. For pruning, we prove linear convergence rate of the nonconvex-nonconcave objective to stationary points under a mild smoothness assumption and a two-sided Polyak-Łojasiewicz condition.
    \item We find that compared with random initialization, the winning lottery ticket found in the supervised setting also achieves higher accuracy in the data-free setting, indicating the Lottery Ticket Hypothesis may transfer across learning methods and has data-dependent benefits to generalization.
\end{itemize}

The proposed method can be widely applied to different network architectures, applications, and datasets.

\section{Related Works}
\paragraph{Data-free Knowledge Transfer}
Overall, image synthesis is a common technique used in many recent methods for accomplishing  tasks such as distillation, network compression, quantization and model inversion in the data-free setting.
\cite{yin2019dreaming} proposes adaptive model inversion to tackle tasks such as pruning, distillation, and continual learning without training data. 
They use a squared error penalty to enforce the batch-statistics of the synthetic images to be similar to those of the training data to generate the synthetic images. 
Apart from the batch-statistics penalty, the image generation/inversion step follows principles of inceptionism \cite{mordvintsev2015inceptionism}. Adaptive model inversion is an enhanced version of DeepInversion~\cite{yin2019dreaming} which aims at increasing diversity by incorporating a loss term in model inversion which maximizes the Jensen-Shannon divergence between the logits of the teacher and student networks. 
\cite{nayak2019zero} samples class labels from a Dirichlet distribution and finds synthetic inputs that minimize the KL Divergence between their outputs in the teacher model and the sampled class labels. 
It then uses such synthetic data for the downstream tasks.

Given a teacher network trained on an unknown dataset, \cite{chen2019data} use a generator to synthesis images that maximize certain responses of the teacher network, so that it can approximate the original training data. Then they use the synthesized images to distill the knowledge of the teacher network onto the student network. \cite{micaelli2019zero} also use a generator, which is trained to generate pseudo data that maximize the output discrepancy between the student and teacher network. This allows the student network to be trained on data spreading over the input space. These methods require full access to the weights and architecture of the teacher network and are not easily applicable to cases where we only have black-box access to the teacher or only know its architecture. 
\cite{fang2019data} trains a generator to generate inputs that maximize the discrepancy between the teacher and student models, while training the student to minimize such discrepancy. 
Despite the similarity in adversarial framework, we train more compact models with quantization and pruning, and provide convergence analysis of such minimax optimization under reasonable assumptions.

\paragraph{Data-free Quantization} \cite{haroush2019knowledge} illustrate that by enforcing a KL penalty on the batch-normalization statistics for image synthesis, one can produce synthetic images which can be used for quantization. 
Similarly, \cite{cai2020zeroq} perform calibration and fine-tuning for quantization by generating synthetic data based on the batch-norm statistics. These batch-statistic-based inversion methods have the limitation that they are targeted for models which are trained with batch-normalization layers. 
In addition to the methods which do quantization by image synthesis, there does exist data-free quantization methods which are post-training. 
\cite{nagel2019data} propose weight equalization and bias correction for data-free quantization. Their proposed method results in minimal loss of ImageNet top-1 accuracy for MobileNetV2 for quantization up to 8-bits ($\approx 0.8 \%$ drop).
To the best of our knowledge, none of the previous \emph{data-free methods} have been able to efficiently train compact netowrks with \emph{binary} weights.

\paragraph{Data-free Compression}
Model compression by pruning, in the conventional setting where we have access to at least a portion of training data, has greatly progressed during recent years. Early works in reducing redundancies in network parameters illustrated that it is possible to reduce the network complexity by removing redundant neurons \cite{srinivas2015data,zhang2010node} and weights \cite{lecun1990optimal}. 
Most pruning methods result in smaller subnetworks with higher accuracy than training the same subnetwork from scratch. 
However, most of the progress has been made under the assumption of data availability, and very few works focus on the data-free setting.
Some recent works~\cite{yin2019dreaming,haroush2019knowledge}  proposed data-free compression by utilizing the batch normalization (BN) statistics~\cite{ioffe2015batch} which store first- and second-order  statistics of the training data. These data-free methods use gradient methods to generate synthetic images which have similar batch statistics to those of the training data by minimizing the distance between the batch statistics of the synthetic images and the stored BN statistics in the trained model, and then directly use the synthetic data for model pruning. 

\paragraph{Lottery Ticket Hypothesis} 
Recently, \cite{frankle2018lottery} proposed the lottery-ticket hypothesis which shows that randomly-initialized dense neural networks contain a much smaller sub-network with proper initialization that have comparable performance to the larger network when trained using the same number of iterations.\footnote{In their experiments, the smaller subnetwork only contained 1.5$\%$ of the $\#$params of VGG-19, and 11.8$\%$ of ResNet-18} 
This smaller subnetwork when initialized with the original initialized values used for training the larger network, achieves comparable accuracy to that of the larger network even when trained, in isolation, from scratch. This sub-network is said to have won the initialization lottery and thus is called the \textit{winning ticket}. 
Unlike the the orignal lottery ticket hypothesis that relies on the availability of training data,  
our focus is on evaluating the transferability of lottery ticket from the supervised setting to the data-free setting.

\section{Data-free Quantization and Pruning}
\subsection{Data-free via Adversarial Training}
Inspired by ~\cite{micaelli2019zero}, we exploit adversarial training in a knowledge distillation setting~\cite{hinton2015distilling} for data-free quantization and pruning. We use the pre-trained large network as the teacher network $T(x;\theta_0)$, and train the compact student network $S(x;\theta_s)$ with quantization or pruning, together with an auxiliary generator $G(z;\theta_g)$. The inputs of the generator $G(z;\theta_g)$ are samples from a Guassian distribution $z\sim \mathcal{N}(0, I)$. The compact network
$S(x;\theta_s)$ is trained to match the output of given network $T(x;\theta_0)$ for any input $x$, while generator $G(z;\theta_g)$ is trained to generate samples that maximize the discrepancy between $S(x;\theta_s)$ and $T(x;\theta_0)$.
The minimax objective is written as
\begin{equation}
    \min_{\theta_s} \max_{\theta_g} \mathbb{E}_{z\sim \mathcal{N}(0, I)} D\left(T(G(z))||S(G(z)) ;\theta_g,\theta_s\right),
\end{equation}
where $D$ is a function that measures the divergence between the predicted class probabilities of the two networks. 
We use $D_{KL}(x||y)=\sum_i x^{(i)} \log (x^{(i)}/y^{(i)})$ for quantization follow \cite{micaelli2019zero}. 
For pruning, we empirically find that the symmetric Jensen-Shannon Divergence $D_{JS}(x||y)=\frac{1}{2} D_{KL}(x||y) + \frac{1}{2}D_{KL}(y||x)$ improves the stability.
Notice this objective is different from~\cite{yin2019dreaming,chen2019data}, where $S(x;\theta_s)$ and $T(x;\theta_0)$ are trained in two separate stages.

In addition, we find the spatial attention regularizations used in \cite{micaelli2019zero,zagoruyko2016paying} is also beneficial for data-free quantization and pruning:
\begin{equation}
\mR_{a}(z;\theta_s) = \beta \sum_{l\in \mathcal{S}_a} \left\lVert \frac{f(s_l)}{\lVert f(s_l) \rVert} - \frac{f(t_l)}{\lVert f(t_l) \rVert} \right\rVert,
\end{equation}
where $\mathcal{S}_a$ is a selected subset of layers, such as the layers before spatial down-sampling operations.
$f(x)={1/N_c} \sum_c (x^{(c)})^2$ computes the spatial attention map as the mean of the squared features over the channel dimension, and $s_l, t_l$ are the feature maps of the student and teacher networks at layer $l$.

For notational convenience, we denote the divergence term as
\begin{equation}
\mD(z;\theta_g, \theta_s)=D\left(T(G(z))||S(G(z)) ;\theta_g,\theta_s\right),
\end{equation}
and the objective function as 
\begin{equation}
    \mathcal{L}(\theta_s,\theta_g)=\mathbb{E}_{z\sim \mathcal{N}(0,I)} \left[\mD(z;\theta_s,\theta_g)+\mR_a (z;\theta_s)\right].
\end{equation}

The minimax problem can be optimized by alternating gradient steps. Note that extra constraints are introduced for quantization (Eq.~\ref{eq:quan_steps}) and pruning (Eq.~\ref{eq:pruning_minimax}). 
We initialize the to-be-quantized compact network by quantizing the full-precision pre-trained weights, and initialize the to-be-pruned network  with pre-trained weights. 
Note that when the compact network is initialized to be exactly the same as the given teacher network, both $\mD(z;\theta_g, \theta_s)$ and $\mR_{a}(z;\theta_s)$ would be zero, which could make initial training steps challenging. 
However, interestingly, the pruning process introduces data-independent regularizations on weights that are not zero (unless all weights are zero), which drives the initial stage of training.

\subsection{Quantization via BinaryConnect}
We slightly modify BinaryConnect (BC)~\cite{courbariaux2015binaryconnect} as the quantization method in the gradient descent steps to update the compact student network. The weights of the network are quantized into binary values $\{-\delta, \delta\}$ during the optimization process following \cite{courbariaux2015binaryconnect,li2017training}, where
$\delta$ is a full-precision scale factor fixed as a constant across all layers. 

We accumulate gradients with a full-precision buffer $\theta_{b}$, and quantize it to get the binary weights. We project the scale of $\theta_{b}$ to be between $-\delta$ and $\delta$ so that the full precision buffer and the binary weights will not diverge. 
In summary, each descent step for the compact network with updated generator $\theta_g^k$ proceeds as following
\begin{itemize}
    \item[1] Compute the gradients from the binary weights by taking the sign of the buffer $\theta_{b}$ as 
    \begin{equation}
        g_k = \nabla_{\theta_s}\mL(\delta\sign(\theta_b^k),\theta_g^k),
    \end{equation}
    \item[2] Accumulate the weight updates into the buffer as
    \begin{equation}
        \hat{\theta}_b^{k+1} = \theta_b^k - \alpha_k g_k,
    \end{equation}
    \item[3] Clip the weights so that it does not exceed the maximum magnitude specified by $\delta$
    \begin{equation}
        \theta_b^{k+1}=\Pi_{\lVert \theta_b \rVert_{\infty}\le \delta} (\hat{\theta}_b^{k+1}),
    \end{equation}
\end{itemize}
where $\alpha_t$ is the learning rate, $\Pi_{\lVert \theta_b \rVert_{\infty}\le \delta}(\cdot)$ is a projection operator on the buffer $\theta_b$ such that its magnitude does not exceed $\delta$. 
Note that we have to keep track of a full precision buffer to quantize to extremely low precision (binary) weights. However, the extra RAM consumption during training is small as the major consumption of RAM comes from the gradient computation. 
After training for $K$ steps, the weights of the binary network is set to 
\begin{equation}
    \theta_s = \delta \sign(\theta_b^K).
\end{equation}

\subsection{Pruning via Sparse Regularization}\label{sec:prune}
We prune the \textit{filters} of convolutional layers so that the pruned network can achieve acceleration on any platform without requiring the hardware to support accelerated sparse operations.
Specifically, let $W\in \sR^{n\times mk^2}$ be the (flattened) weight matrix of any convolutional layer, with $n$ output channels, $m$ input channels, and a kernel size of $k$.

\begin{wrapfigure}{r}{0.5\textwidth}
\centering
  \includegraphics[width=0.36\textwidth]{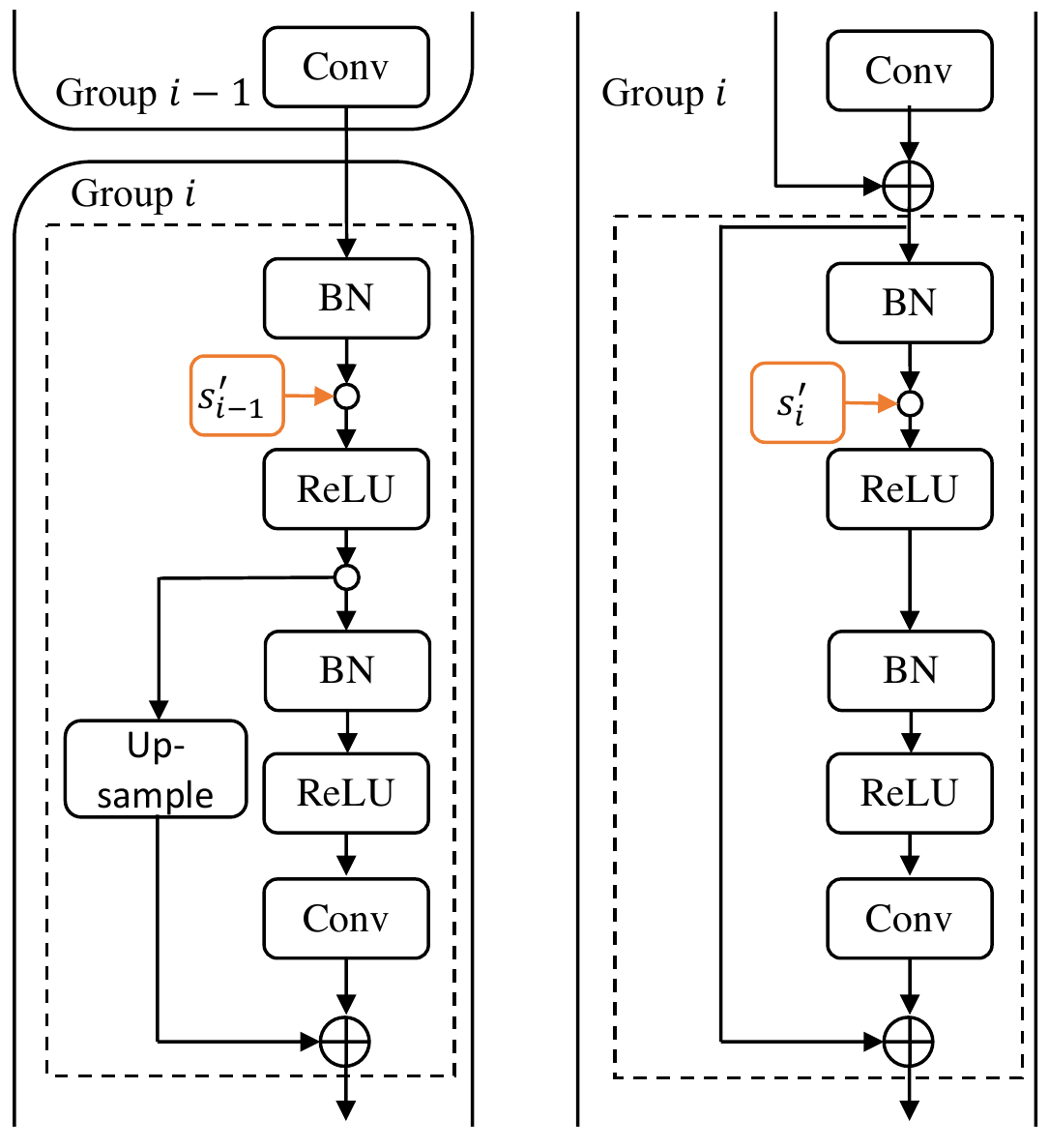}
  \caption{\small Sharing the scaling factor $s$ between residual blocks with same number of output channels for the pre-activation residual connections used by the networks in this paper. We group the residual blocks according to the number of output channels. Inside the dashed rectangles are two types of residual blocks, where the first one containing up-sampling operation in the residual connections is the first block of each group. Its first BN is followed by a scaling factor $s'_{i-1}$ shared with the last group. For the other types of residual blocks, their first BN is followed by $s_i'$ shared inside the group.}
  \label{fig:ressharing} 
\end{wrapfigure}
Inspired by \cite{liu2017learning}, we introduce a trainable scaling factor scaling factor $s\in \sR^n$ for each convolutional filter, i.e., each of the $n$ filters of $W$ and the $n$ entries of the bias (if any) is multiplied by $s\in \sR^n$. This is equivalent to multiplying each channel of the output feature map of the convolution operation by $s$. 
For layers with Batch Normalization (BN)~\cite{ioffe2015batch}, we can use the trainable scaling factor introduced by BN, and change BN into the following equivalent form:
\begin{equation}
    y = s\left(\frac{x-\mu}{\sqrt{\sigma^2+\epsilon}}+b\right),
\end{equation}
where $x$ is the input batch of features, $\mu,\sigma^2$ are the mean and variance of the batch, $\epsilon>0$ is a constant which prevents division by zero, and $s, b$ are trainable parameters in BN. 

We enforce sparsity of these scaling factors $s$ by adding an $\ell_1$-norm regularization on $s$, assuming that the number of necessary filters are less than the pre-defined redundant structure of the large network. 
Together with the regularization from weight decay, redundant filters for the task will be guided to have small weights, and can be identified by the corresponding scaling factor $s$ since removing filters with small magnitudes will not have much effect on the final feature representations. 
After the training process, we set a threshold $t_s$, and convolutional filters with small trainable scaling factors $s<t_s$ will be pruned.

More specifically, we add the following sparse regularization to the original loss function $\mL(\theta_s, \theta_g)$ for pruning:
\begin{equation}\label{eq:prune_reg}
    \mR_{p}(\theta_s)=\sum_{l=1}^{L} \gamma_l \lVert s_l \rVert_1 + \lambda \lVert W_l \rVert_F^2 + \lambda\lVert b_l \rVert^2,
\end{equation}
where $l$ is the index of the layer, $\gamma_l$ and $\lambda$ are constants.
The values of $\gamma_l$ are decided by the size of the feature map. 
In multi-layer convolutional networks, feature maps with larger spatial sizes typically have fewer number of channels and each feature map potentially carries more information. 
Hence we set $\gamma_l=\gamma / w_l$, where $w_l$ is the width of the feature map in layer $l$ and $\gamma$ is a constant. 

For residual blocks in modern convolutional networks, pruning is more efficient when the corresponding pruned features maps are aligned for layers connected by the residual connection. We apply shared scaling factors for the entire residual block to avoid potential inconsistency between convolutional layers within the residual blocks.

\section{Convergence analysis}
In this section, we analyze the convergence of the alternating optimization for solving the minimax problem under quantization and pruning constraints.
This fills in the blank of theoretical analysis for previous data-free/zero-shot knowledge transfer methods which utilize a generator to generate the synthetic data~\cite{fang2019data,micaelli2019zero}.
To make the conclusions applicable to a broader class of problems, by an abuse of notation, we use $\mF$ to denote the objective function satisfying certain properties, instead of the loss functions $\mL$ for the specific problems.

\subsection{Data-free Quantization}
In data-free quantization, we are solving the following minimax problem,
\begin{equation}
\min_x \max_y \mF(x, y) \label{eq:prob}
\end{equation}
by the stochastic update rule
\begin{equation}\label{eq:quan_steps}
\begin{split}
\hat x \kp &= \hat x_k - \alpha_k g_x (x_k, y_k) \\
x\kp & = \mQ (\hat x \kp)\\
y\kp &= y_k + \beta_k g_y (x\kp, y_k) 
\end{split}
\end{equation}
where $\bbE g_x (x, y) = \nabla_x \mF(x, y), \bbE g_y (x, y) = \nabla_y \mF(x, y) $, $\alpha_k, \beta_k$ are stepsizes, and $\mQ$ is the quantization function $\mQ=\delta\sign(x)$. 

Assume the optimal solution  $(x\opt , y\opt)$ exists, then $\nabla_x \mF(x\opt , y) = \nabla_y \mF(x, y\opt) = 0$. 
The following theorem illustrates the convergence of this method by stating that the duality gap, $P(x_k, y_k) = \mF(x_k, y\opt) - \mF (x\opt, y_k)$, vanishes.

\begin{theorem}\label{thm:quantization_main}
Suppose the function $\mF (x, y)$ is convex in $x$, concave in $y$, and Lipschitz (i.e., $\| \mF(x_1, y) - \mF(x_2, y)\| \leq L \| x_1 - x_2 \|  $); and that the partial gradients are uniformly Lipschitz smooth in $x$, ($i.e., \| \nabla_x \mF(x_1, y) - \nabla_x \mF(x_2, y)\| \leq L_x \| x_1 - x_2 \|  $,  $\| \nabla_y \mF  (x_1,y)- \nabla_y \mF(x_2, y)\| \le L_y\|x_1-x_2\|$). Suppose further that the stochastic gradient approximations satisfy $\mathbb{E} \|g_x(x, y)\|^2\le G_x^2,$ $\mathbb{E} \| g_y(x, y)\|^2\le G_y^2$ for scalars $G_x$ and $G_y,$ and that $\mathbb{E} \|x^k-x\opt\|^2\le D_x^2,$ and $\mathbb{E} \|y^k-y\opt\|^2\le D_y^2$   for scalars $D_x$ and $D_y.$

If we choose decreasing learning rate parameters of the form $\alpha_k=\frac{C_\alpha}{\sqrt{k}}$ and  $\beta_k = \frac{C_\beta}{\sqrt{k}},$ then the alternating optimization has the error bound,
\begin{equation}
\begin{split}
  &  \mathbb{E}[P(\bar x^l, \bar y^l)] \\
    \leq & \frac{1}{2\sqrt{l}} \left(\frac{D_x^2}{C_\alpha} + \frac{D_y^2}{C_\beta}\right) + \frac{\sqrt{l+1}}{2l} \left( \right.\\
    &\left. C_\alpha G_x^2 + C_\alpha L_y G_x^2 + C_\alpha L_y D_y^2 + C_\beta G_y^2\right) \\
    & + (L_x D_x+LD_x + 2L_yD_y)\sqrt{d} \Delta
\end{split}
\end{equation}
where $\bar x^l = \frac{1}{l} \sum_{k=1}^l x^k, \, \bar y^l = \frac{1}{l} \sum_{k=1}^l y^k.$
\end{theorem}

From Theorem~\ref{thm:quantization_main}, we can see that the error bound decreases with a standard $O(1/\sqrt{k})$ convergence rate with the convex-concave assumption for stochastic alternating optimization, and eventually converges to a region characterized by the quantization grain $\Delta$. 
Our error bound also suggests that though we only quantize the compact network $x$, the smoothness of the partial gradients of the compact network $\nabla_x \mF$ and the teacher network  $\nabla_y \mF$ will reflect in the quantization error. 
Although the convex-concave assumption is a widely used assumption cannot be satisfied by neural networks in practice, our result provide useful insights and fills in the blank of such analysis in data-free quantization.

\subsection{Data-free pruning}
Formally, in the case of data-free pruning, we are solving the following  minimax problem with sparse rank-reduced regularization,
\begin{equation}
\min_x \max_y \mF(x, y), \label{eq:prune_prob}
\end{equation}
where $x=\theta_s,y=\theta_g$, and $\mF(x,y)=\bbE_{z\sim \mN(0,I)}[\mD(z;x,y)+\mR_a(z;x)+\mR_p(x))]$. 
Note that this analysis is general and can be directly applied to previous data-free methods without quantization or pruning, e.g., \cite{micaelli2019zero}. We are unaware of a previous theoretical analysis for such data-free methods. 
We assume the gradient $\nabla \mF$ can be obtained directly, and use the following updates
\begin{equation}\label{eq:pruning_minimax}
\begin{split}
 x\kp &= x_k - \alpha_k \nabla_x \mF (x_k, y_k) \\
y\kp &= y_k + \beta_k \nabla_y \mF  (x\kp, y_k). 
\end{split}
\end{equation}

For the above scheme, we can prove convergence for a class of nonconvex-nonconcave functions $\mF$ in the sense that the gradients vanish, and the method approaches as stationary point. 
Note, this is stronger than the duality gap notion of convergence used for Theorem 1.
Such class of functions should satisfy the following three assumptions.
\begin{assumption}[$L$-Lipschitz gradient/$L$-Smooth]\label{assumption1}
We say $\mF(x,y)$ has $L$-Lipschitz gradient, or equivalently $L$-smooth, if there exists a positive constant $L>0$ such that 
\begin{equation*}
\begin{split}
    \lVert \nabla_x \mF(x_1,y_1) - \nabla_x \mF(x_2,y_2) \rVert  &\le L[\lVert x_1 - x_2 \rVert + \lVert y_1 - y_2  \rVert ], \\
    \lVert \nabla_y \mF(x_1,y_1) - \nabla_y \mF(x_2,y_2) \rVert  &\le L[\lVert x_1 - x_2 \rVert + \lVert y_1 - y_2  \rVert ].
\end{split}
\end{equation*}
\end{assumption}

\begin{assumption}[Existence of Stationary Point]
The objective function $\mF$ has at least one stationary point $(x\opt,y\opt)$ where $\lVert\nabla_x \mF(x\opt,y\opt)\rVert=\lVert\nabla_y \mF(x\opt,y\opt)\rVert=0$. 
Also, assume for any fixed $y$, $\arg\min_x \mF(x,y)$ is a non-empty set with finite optimal values, and $\arg\max_y \mF(x,y)$ is a non-empty set with finite optimal values.

\end{assumption}
\begin{assumption}[Two-sided PL condition~\cite{yang2020global}]\label{assumption3}
The objective function $\mF(x,y)$ satisfies the two-sided PL condition if there exists constants $\mu_1, \mu_2 > 0$ such that 
\begin{equation*}
\begin{split}
    \frac{1}{2}\lVert \nabla_x \mF(x,y) \rVert^2 &\ge \mu_1[\mF(x,y)-\min_x \mF(x,y)], \forall x,y, \\
    \frac{1}{2}\lVert \nabla_y \mF(x,y) \rVert^2 &\ge \mu_2[\max_y \mF(x,y)-\mF(x,y)], \forall x,y. \\
\end{split}
\end{equation*}
\end{assumption}
Notice that two-sided PL condition does not imply convexity-concavity. 
The objective function $\mF$ can still be nonconvex-nonconcave, as is the case for neural networks.

Also, define the following potential function to measure the inaccuracy of $(x_k,y_k)$
\begin{equation}
    P_k := a_k + \lambda b_k,
\end{equation}
where $a_k = h(x_k) - h^*$, $b_k = h(x_k) - \mF(x_k, y_k)$, $h(x)=\max_{y} \mF(x,y)$, and $h^*=\min_x h(x)$. 
Notice both $a_k$ and $b_k$ are non-negative. 

With these assumptions, we prove linear convergence of the objective function to its stationary point with the update rules in Eq.~\ref{eq:pruning_minimax}. 
We give the proof in the supplementary material.
The proof technique follows~\cite{yang2020global}.

\begin{theorem}[Linear Convergence to Stationary Point]\label{thm:grad_convergence}
Suppose $\mF(x,y)$ satisfies Assumptions 1,2,3. Define $P_k=a_k + \frac{1}{10}b_k$, $L_h=L+\frac{L^2}{2\mu_2}$. If we run the updates in Eq.~\ref{eq:pruning_minimax} with $\alpha=\frac{\mu_2^2}{18L^3}$ and $\beta=\frac{1}{L}$, then
\begin{equation}
     \lVert \nabla_x \mF(x_k,y_k) \rVert^2 + \lVert \nabla_y \mF(x_k,y_k) \rVert^2  \le P_0 M \left(1-\frac{\mu_1\mu_2^2}{36L^3}\right)^{k},
\end{equation}
where $M=\max\{\frac{2L_h^2}{\mu_1}, \frac{40L^2}{\mu_2}\}$.
\end{theorem}
Note that due to the $\ell_1$ regularization in the pruning objective (Eq.~\ref{eq:prune_reg}), and the popular choice of ReLU activation for convolutional neural networks, the smoothness assumption (Assumption~\ref{assumption1}) is not satisfied in our settings. 
However, the assumption can be satisfied if we use an objective without $\ell_1$ regularization such as VIBNet~\cite{dai2018compressing}, accompanied by smooth activations such as GELU~\cite{hendrycks2016gaussian} which approximates ReLU quite well.

\section{Data-free Lottery Ticket}\label{sec:lottery}
We provided theoretical analysis for the convergence of the data-free approach to stationary points for the general adversarial training framework, as well as with quantization and pruning.
However, it remains unclear why the data-free training process could land in a good solution where the compact networks generalize and perform well on real world data. 

The images synthesized by a well-trained generator are far from being similar to the actual data, e.g., the visualizations of \cite{micaelli2019zero}.

One of the key factors contributing to the success of deep learning is the effective parameter initialization, in the sense that the distribution of the initial weights is good~\cite{he2015delving}, or there exists a good subnetwork with a proper structure and initialization (the winning Lottery Ticket)~\cite{frankle2018lottery}. 

We investigate whether there exists a ``universal'' winning ticket that is beneficial for both supervised and data-free settings by finding the winning ticket in one setting ($\mS_1$) and evaluating it in another setting ($\mS_2$). 
Let $m\in\{0,1\}^{|\theta_s|}$ be a mask for the student's set of parameters $\theta_s$ to denote which connections are in the winning ticket. 
To find a sparsely connected winning ticket $m\odot \theta_s$ with $p\%$ fewer connections, we use an iterative pruning approach similar to~\cite{frankle2018lottery} and repeat the following procedure for $n$ rounds in setting $\mS_1$ before the winning ticket is re-trained with the same initialization in setting $\mS_2$:
\begin{enumerate}
    \item[1] If it is the first round, randomly initialize $\theta_s = \theta_s^{(0)}$, initialize all entries of $m$ to 1, and store the value of $\theta_s^{(0)}$. Otherwise, set $\theta_s=m\odot \theta_s^{(0)}$. 
    \item[2] Train the network $S(x;\theta_s\odot m)$ in setting $\mS_1$ for $K$ iterations, only update weights with $m_i=1$, and get a network parameterized by $m\odot \theta_s^{(K)}$.
    \item[3] Sort the unpruned connections by the absolute values of their weights. Prune $(p\%)^{1/n}$ of the weights with smallest absolute values by setting the corresponding masks to 0.
\end{enumerate}
In the current version, we always let $\mS_1$ be the data-free setting and $\mS_2$ be the supervised setting. 
As it will be shown in the experiments, we empirically show that winning tickets obtained from the supervised setting is beneficial for the data-free setting and achieves higher test accuracy, 
which indicates that certain subnetworks emerge at initialization and helps improve the generalization for both supervised and data-free setting. We demonstrate the transferability of winning tickets and provide additional evidence that the winning tickets are intrinsic properties of neural networks.

\section{Experiments}
We follow the experimental setting in data-free knowledge distillation \cite{micaelli2019zero}, where WRN-40-2 and WRN-16-2 networks \cite{zagoruyko2016wide} are pre-trained on the CIFAR-10 dataset \cite{krizhevsky} as teacher networks. We train compact networks by the proposed data-free quantization and pruning methods, and report the accuracy on the test set of CIFAR10. We also present the size of the networks, and compare with baselines of data-free methods and fine-tuning methods with data. We perform ablation study on hyperparameters of the proposed methods. 

All experiments run on a single GPU (2080 Ti).
Unless otherwise specified, we use the default settings following \cite{micaelli2019zero}. 
We use Adam optimizer for both the compact network and the auxiliary generator, with learning rates of $2\times10^{-3}$ and $10^{-3}$, respectively, and a batch size of 128.
The dimension of the generator input $z$ is 100.
The generator takes 1 gradient ascent step to increase $\mL(\theta_s,\theta_g)$, followed by the student taking 10 gradient descent steps (followed by clipping $\theta_b$ for quantization) on the synthetic batch generated by the generator to decrease $\mL(\theta_s,\theta_g)$ (plus regularization terms $\mR_p$ for pruning).

\subsection{Data-free Quantization}
\begin{wrapfigure}{r}{0.5\textwidth}
\centering
  \includegraphics[width=0.75\linewidth]{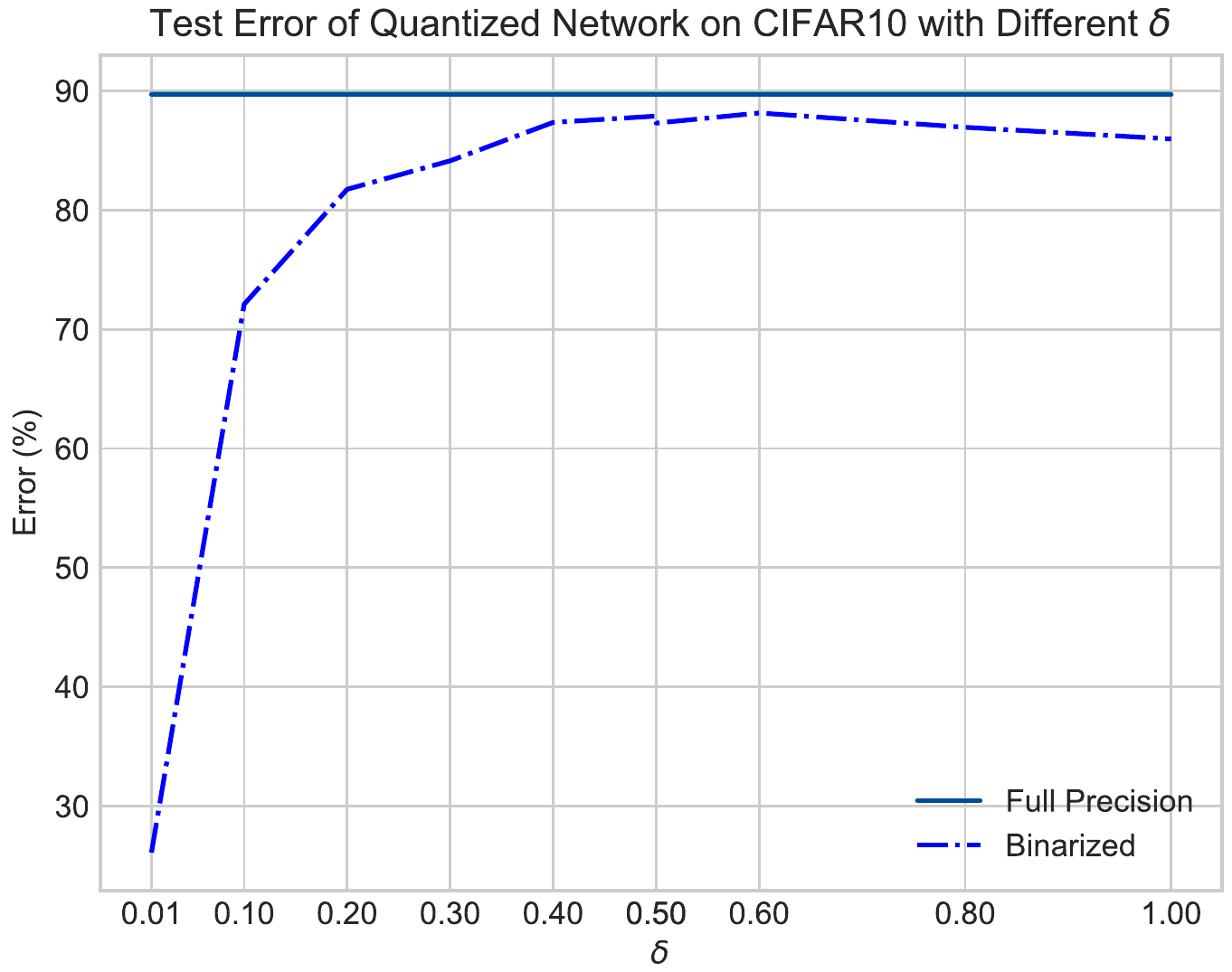}
  \caption{\small Accuracies of the quantized WRN-16-2 under different scales $\delta$, using WRN-40-2 as the teacher network. The solid line is the result of training a full-precision WRN-16-2 under the same hyperparameters in the data-free setting~\cite{micaelli2019zero}. The highest accuracy after quantization is 88.14\%, while the full precision one has an average accuracy of 89.71\%.}
  \label{fig:delta_acc} 
\end{wrapfigure}

Following the practice of~\cite{rastegari2016xnor}, we leave the first convolutional layer and the final linear layer's weights as full-precision, and quantize the weights of all intermediate layers into binary.
We fix the scaling factor across all quantized layers as a constant $\delta$. 
We do a grid search for $\delta$ by setting the teacher to WRN-40-2 and the student to WRN-16-2, and present the results in Figure~\ref{fig:delta_acc}.
The accuracy of our best result (88.14\%) is only 1.57\% lower then that of the full-precision network in the data-free setting \cite{micaelli2019zero}, even though most of the networks weights are binarized. 
Further, by using WRN-16-2 as the teacher network, and initializing the weights of the binary student as $\delta \sign(\theta_0)$, we can achieve a higher accuracy of 88.98\%.
As a comparison, training a WRN-16-2 with BC on CIFAR10 achieves 92.97\% in the presence of data with data augmentation.
This indicates our data-free framework is able to recover most of the capabilities of augmented data for training binary networks.

\subsection{Data-free Pruning}
\begin{table}[htbp!]
\centering
\resizebox{0.5\linewidth}{!}{%
\begin{tabular}{l  c  c  c  c  c }
\toprule
\footnotesize{$\lambda$}       & \footnotesize{1e-5} & \footnotesize{2e-5}    & \footnotesize{4e-5} & \footnotesize{5e-5}    & \footnotesize{6e-5} \\ \hline
\footnotesize{\#Params}  & \footnotesize{570K}   & \footnotesize{408K} & \footnotesize{348K}   & \footnotesize{337K} & \footnotesize{323K }   \\
\footnotesize{\#FLOPs}   & \footnotesize{82.7M}   & \footnotesize{68.5M} & \footnotesize{54.9M}   & \footnotesize{50.3M} & \footnotesize{47.7M }   \\
\footnotesize{Acc (\%)} & \footnotesize{92.77}   & \footnotesize{92.19} & \footnotesize{90.79}   & \footnotesize{89.92} & \footnotesize{89.17}   \\
\bottomrule
\end{tabular}
}
\caption{\small Performance of the pruned network (WRN-16-2) under different weight decay ($\lambda$) when $\gamma$=2e-3 and the pruning threshold $t_s=0.1$.}
\label{tab:prune_wd}
\vspace{-0.5em}
\end{table}
For pruning, we have introduced additional hyper-parameters $\gamma, \lambda$ as defined in Eq.~\ref{eq:prune_reg}, and we use Jensen-Shannon divergence $D_{JS}$.
Following a grid search, we set the learning rate to $10^{-3}$. We prune a WRN-16-2 network. 
Table~\ref{tab:prune_wd} shows the number of parameters (\#Params), floating point operations (\#FLOPs) and the accuracy (Acc) of the pruned model under different weight decays ($\lambda$), from which we can see that weight decay has significant impact on the size of the pruned model. Increasing the weight decay penalty by a factor of 6, results in a network with 43\% fewer parameters. 

By comparison, the impact of $\gamma$ is less significant in the observed range.
If we fix $\lambda=$4e-3, \#Params are 372K, 348K and 333K for $\gamma=$1e-3, 2e-3 and 4e-3, respectively. 
However, higher weight decay can lead to instability of the training process.

The value of weight decay $\lambda$ not only affects the compression ratio, but also affects the quality of the generator. 
We plot the images generated by the generator at the end of the optimization process in Figure~\ref{fig:generated}. 
As $\lambda$ goes higher, the compression ratio is higher and the generated images become sharper.

\begin{figure}[!htbp]
\centering
  \includegraphics[width=\linewidth]{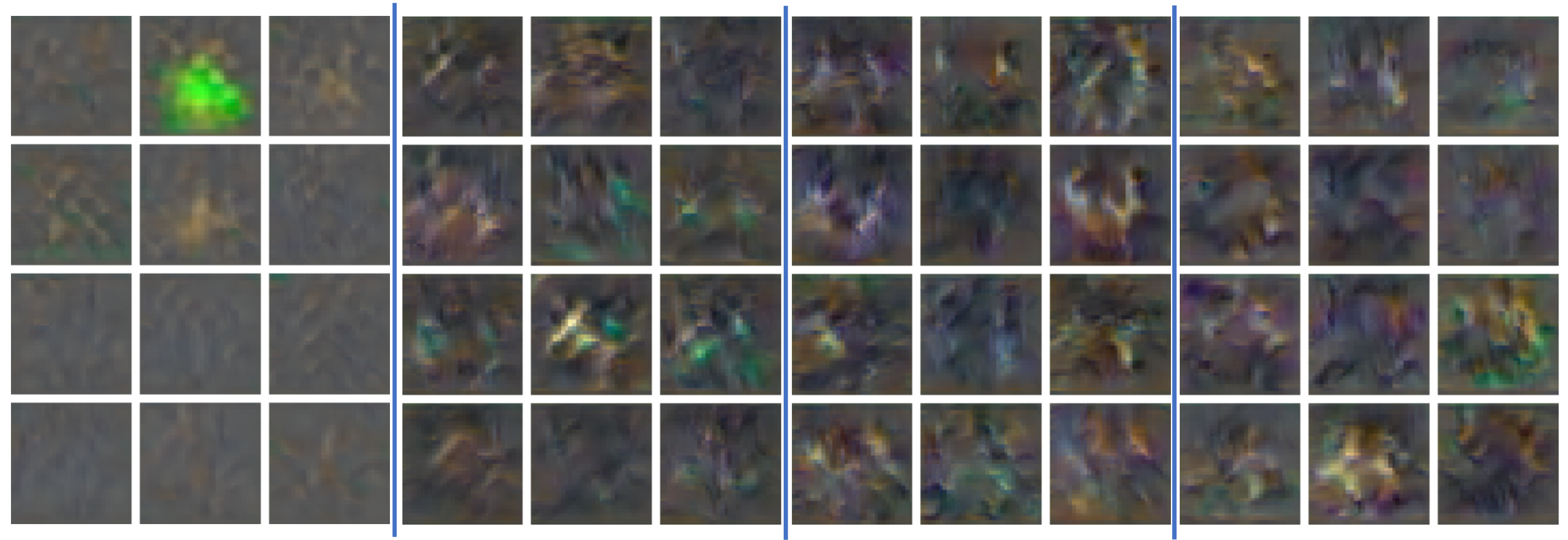}
  \caption{\small Random samples from the generator. Every three columns correspond to the generator from a different setting, corresponding to the settings for $\lambda=$1e-5, 2e-5, 4e-5, 5e-5 in Table~\ref{tab:prune_wd}. As $\lambda$ becomes larger, the network is pruned further and the generated images look sharper.}
  \label{fig:generated} 
  \vspace{-1em}
\end{figure}

\textbf{Comparing two divergence metrics:}
We find using the symmetric Jensen-Shannon divergence (JSD) as the objective results in better compression ratios and improves the stability. 
To analyze what contributes to such an improvement, we look at the average entropy of the teacher network's predictions throughout the training process. 
The lower the entropy is, 
the teacher network's output probabilities for the input pseudo batches generated by $G(z;\theta_g)$ are more concentrated, which indicates that the pseudo batches are closer to the 
teacher network's training data.
In fact, using JSD does reduce such entropy under the same setting, as shown in Table~\ref{tab:div}. 
\begin{table}[htbp!]
\centering
\resizebox{0.7\linewidth}{!}{%
\begin{tabular}{l  c  c  c  c  c }
\toprule
                                   & \footnotesize{\#Params } & \footnotesize{Acc (\%)}    & \footnotesize{\#Params } & \footnotesize{Acc (\%)}    & \footnotesize{Entropy} \\ \hline
\scriptsize{KLD}\scriptsize{($\gamma$=1$\mathrm{e}$-3, $\lambda$=2$\mathrm{e}$-5)}  & \footnotesize{462K}   & \footnotesize{92.89} & \footnotesize{450K}   & \footnotesize{92.88} & \footnotesize{0.93}   \\
\scriptsize{KLD}\scriptsize{($\gamma$=2$\mathrm{e}$-3, $\lambda$=2$\mathrm{e}$-5)}  & \footnotesize{478K}   & \footnotesize{92.75} & \footnotesize{451K}   & \footnotesize{73.48} & \footnotesize{1.01}   \\
\scriptsize{JSD}\scriptsize{($\gamma$=1$\mathrm{e}$-3, $\lambda$=2$\mathrm{e}$-5)} & \footnotesize{453K}   & \footnotesize{92.49} & \footnotesize{445K}   & \footnotesize{92.49} & \footnotesize{0.87}   \\
\scriptsize{JSD}\scriptsize{($\gamma$=2$\mathrm{e}$-3, $\lambda$=2$\mathrm{e}$-5)} & \footnotesize{429K}   & \footnotesize{92.19} & \footnotesize{408K}   & \footnotesize{92.19} & \footnotesize{0.87}   \\
\bottomrule
\end{tabular}
}
\caption{\small Comparing the KL divergence (KLD) and symmetric Jensen-Shannon divergences (JSD) for pruning the WRN-16-2 model under similar settings. The first two columns of the results are obtained when setting the pruning threshold $t_s=0.01$, while the following two are setting $t_s=0.1$.
Larger $\lambda$ and $\gamma$ can lead to higher sparsity and compression ratio, but the accuracy of using KL breaks down to 73.5\% when $\gamma$ is increased from 1e-3 to 2e-3. The compression ratio with SKL also tends to be higher.}
\label{tab:div}
\vspace{-0.5em}
\end{table}
With JSD, the generator generates pseudo batches that are closer to the data distribution for the following compression. 

\textbf{Comparison to supervised setting:}
In the supervised setting, various data augmentation techniques can be applied to improve the generalization of the model. 
Counter-intuitively, such data augmentations can also improve the compression ratio of the WRN's for our compression approach. 
Our data-free pruning does not perform as good as the supervised setting with data augmentation, but is quite close to the supervised setting without data augmentation. For instance, we can prune the network down to $\approx 250K$ parameters while maintaining $\approx 90\%$ accuracy. The results are shown in Table~\ref{tab:prune_sup}. 
This inspires us to further investigate enhancing the variety of the generated pseudo batches. 

\begin{table}[htbp!]
\centering
\resizebox{0.6\linewidth}{!}{%
\begin{tabular}{l  c  c  c  c  c }
\toprule
                          & \footnotesize{$\lambda$} & \footnotesize{$\gamma$}    & \footnotesize{\#Params} & \footnotesize{\#FLOPs}    & \footnotesize{Acc} \\ \hline
\footnotesize{Supervised}       & \footnotesize{5e-4} & \footnotesize{1e-3}    & \footnotesize{243K} & \footnotesize{36.8M}    & \footnotesize{88.84} \\ 
\footnotesize{Supervised + Aug.}  & \footnotesize{5e-4}   & \footnotesize{1e-3}
& \footnotesize{236K}   & \footnotesize{43.3M} & \footnotesize{92.71}   \\
\hline
\footnotesize{Data-free}   & \footnotesize{5e-5}   & \footnotesize{3e-3} & \footnotesize{339K}   & \footnotesize{52.3M} & \footnotesize{90.57}   \\
\footnotesize{Data-free + Warm up}   & \footnotesize{3e-4}   & \footnotesize{1e-2} & \footnotesize{254K}   & \footnotesize{48.2M} & \footnotesize{89.19}   \\
\bottomrule
\end{tabular}
}
\caption{\small Comparing the data free approach with supervised setting, where in the supervised setting the training data is available. For ``Data-free + Warm up", we increase the value of $\lambda$ and $\gamma$ linearly from 0 to the values reported in the table. With data augmentation, even more parameters can be pruned in the supervised setting, despite having more computations for a higher test accuracy. The data free setting preserves slightly more parameters than the supervised setting without data augmentations, but the test accuracy is higher. Note that the unpruned network in the same setting has an accuracy of 89.71\%~\cite{micaelli2019zero}.}
\label{tab:prune_sup}
\end{table}
\begin{table}[htbp!]
\centering
\resizebox{0.5\linewidth}{!}{%
\begin{tabular}{ c  c  c  c }
\toprule
 \footnotesize{Student}    & \footnotesize{\#Params (ticket)} & \footnotesize{Acc. (ticket)}    & \footnotesize{Acc. (random)} \\ \hline
\footnotesize{WRN-16-2} & \footnotesize{53.0K}   & \footnotesize{77.03} & \footnotesize{75.48 }   \\
\footnotesize{WRN-16-2} & \footnotesize{45.5K}   & \footnotesize{74.66} & \footnotesize{73.56 }   \\
\bottomrule
\end{tabular}
}
\caption{\small Comparing the test accuracy of training the winning ticket found in the supervised setting, and the network with the same structure but a different random initialization. We prune 20\% weights of the convolutional layers, and run 13 and 14 rounds to find the two tickets.}
\label{tab:lottery}
\vspace{-1em}
\end{table}

\subsection{Finding Winning Tickets for Data-free Setting}
To find the winning lottery ticket, we use the procedure as described in the previous section, finding the lottery ticket on supervised setting and evaluate the winning ticket in the data-free setting by training only the weights from the winning ticket in our framework. 
Following the same setup as~\cite{frankle2018lottery}, we only prune the parameters of the weights of convolutional layers. 
In each round, we prune 20\% of the remaining weights. 
We use a batch size of 128, a learning rate of 0.03 with SGD (momentum 0.9) and train the network for 30000 iterations in the supervised settings.
For the data-free setting, the is set to WRN-40-2. 
The results are in Table~\ref{tab:lottery}. 
From the results of~\cite{frankle2018lottery}, we have already know that such winning tickets are beneficial for supervised learning. 
Despite using a different optimizer (Adam), a different learning rate (2e-3), and a different learning approach (data-free), such winning tickets is still beneficial for the network to generalize well on the same dataset, indicating that the winning ticket has some inductive bias which benefits generalization and transfers across optimizers and learning approaches. 

\section{Conclusion}
We study data-free quantization and pruning for training compact networks with strong performance, and provide empirical and theoretical analysis for the proposed adversarial training method. 
To the best of our knowledge, this paper presents the first method that can train compact network with extreme low bit precision, i.e., \textit{binary} quantization,  without having access to training data.
Empirically, we show that weight decay has a great effect on accuracy and the compression ratio. We also illustrate that a symmetric divergence such as Jensen-Shannon enhances the quality of the synthetic input examples.
We provide convergence guarantees under mild conditions for the general minimax problem underlying the data-free adversarial training framework, with and without extra constraints from quantization and pruning. 
Finally, we demonstrate the transferability of Lottery Tickets by showing that winning tickets from the standard supervised setting can benefit the data-free training, shedding some light on the connections in optimziation landscapes between supervised and the proposed data-free learning.

\bibliographystyle{alpha}
\bibliography{refs}

\newpage
\appendix

\section{Proofs for Data-free Quantization}
\renewcommand{\eqref}{Eq.~\ref}

Assume the optimal solution  $(x\opt , y\opt)$ exists, then $\nabla_x \mF(x\opt , y) = \nabla_y \mF(x, y\opt) = 0$. We show the convergence for the primal-dual gap $P(x_k, y_k) = \mF(x_k, y\opt) - \mF (x\opt, y_k)$. 
We prove the $O(1/\sqrt{k})$ convergence rate in Theorem \ref{thm:quantization_main} by using  Lemma \ref{lm1} and Lemma \ref{lm2}, which present the contraction of primal and dual updates, respectively. 

\begin{lemma}
The quantization error is bounded by 
\begin{equation}
    \| \mQ(x) - x \| \leq \sqrt{d} \Delta
\end{equation}
\end{lemma}

\begin{lemma} \label{lm1}
Suppose $\mF(x,y)$ is convex in $x$ and Lipschitz $\| \mF(x_1, y) - \mF(x_2, y)\| \leq L \| x_1 - x_2 \|  $; and has Lipschitz gradients $\| \nabla_x \mF(x_1, y) - \nabla_x \mF(x_2, y)\| \leq L_x \| x_1 - x_2 \|  $; and bounded variance $\bbE [\| g_x (x, y)\|^2] \leq G_x^2$; and $\bbE[\| x_k -x^*\|^2] \leq D_x^2$ , we have 
\begin{equation} \label{eq:lm1}
\begin{split}
\bbE[\mF( x_k, y_k)] - \bbE[\mF(x\opt, y_k)]   \leq &\frac{1}{2\alpha_k}\left(\bbE[\| \hat x_k - x\opt \|^2] -\bbE[\| \hat x\kp - x\opt \|^2] \right) \\
&+ \frac{\alpha_k}{2} G_x^2  + (L_x+L) D_x \sqrt{d} \Delta
\end{split}
\end{equation}
\end{lemma}

\begin{proof}
From gradient descent step , we have 
\begin{equation}
\begin{split}
& \| \hat x\kp - x\opt \| ^2 \\
= & \|  \hat x_k - \alpha_k g_x(x_k, y_k) - x\opt \| ^2 \\
= & \| \hat x_k - x\opt\|^2  - 2 \alpha_k \, \inprod{ g_x(x_k, y_k), \, \hat x_k - x\opt} + \alpha_k^2 \, \|  g_x(x_k, y_k)\|^2 \\
= & \| \hat x_k - x\opt\|^2  - 2 \alpha_k \, \inprod{ g_x(\hat x_k, y_k) -g_x(\hat x_k, y_k) + g_x(x_k, y_k), \, \hat x_k - x\opt} + \alpha_k^2 \, \|  g_x(x_k, y_k)\|^2 \\
= & \| \hat x_k - x\opt\|^2  - 2 \alpha_k \, \inprod{ g_x(\hat x_k, y_k) , \, \hat x_k - x\opt} + 2 \alpha_k \, \inprod{ g_x(\hat x_k, y_k) - g_x(x_k, y_k), \, \hat x_k - x\opt} \\
&+ \alpha_k^2 \, \|  g_x(x_k, y_k)\|^2 
\end{split} 
\end{equation}
Take expectation on both side of the equation, $\inprod{ g_x(\hat x_k, y_k) - g_x(x_k, y_k), \, \hat x_k - x\opt}$ on the right hand side can be written as
\begin{align}
    & \bbE[\inprod{ g_x(\hat x_k, y_k) - g_x(x_k, y_k), \, \hat x_k - x\opt}] \\
    = & \bbE[\inprod{\nabla_x \mF(\hat x_k, y_k) - \nabla_x \mF(x_k, y_k), \, \hat x_k - x\opt}] \\
    \leq & \bbE[\|\nabla_x \mF(\hat x_k, y_k) - \nabla_x \mF(x_k, y_k)\| \, \|\hat x_k - x\opt\|] \\
     \leq & \bbE[L_x \| \hat x_k - x_k \| \, \|\hat x_k - x\opt\|] \\
     \leq & L_x \sqrt{d} \Delta \bbE[\|\hat x_k - x\opt\|]
\end{align}
Substitute with $\bbE[g_x (x,y)] = \nabla_x \mF(x,y)$, apply $\bbE[\| g_x^2 (x,y)\|] \leq G_x^2$ and $\bbE[\|\hat x_k - x\opt\|] \leq \sqrt{\bbE[\|\hat x_k - x\opt\|^2]} = D_x$ to get
\begin{equation}
\begin{split}
\bbE[\| x\kp - x\opt \| ^2] \leq & \bbE[\| x_k - x\opt\|^2]  - 2 \alpha_k \, \bbE[\inprod{ \nabla_x \mF(x_k, y_k), \, x_k - x\opt}] \\
& + \alpha_k^2 G_x^2 + 2\alpha_k L_x D_x \sqrt{d} \Delta. \label{eq:lm1tmp1}
\end{split} 
\end{equation}
Since $\mF(x,y)$ is convex in $x$, we have 
\begin{equation}\label{eq:lm1tmp2}
\inprod{ \nabla_x \mF(x_k, y_k), \, x_k - x\opt} \geq \mF(x_k, y_k) - \mF(x\opt, y_k).
\end{equation}
Combining \eqref{eq:lm1tmp1} and \eqref{eq:lm1tmp2}, we have 
\begin{equation} 
\begin{split}
\bbE[\mF(\hat x_k, y_k)] - \bbE[\mF(x\opt, y_k)]   \leq &\frac{1}{2\alpha_k}\left(\bbE[\| \hat x_k - x\opt \|^2] -\bbE[\| \hat x\kp - x\opt \|^2] \right) \\
 &+ \frac{\alpha_k}{2} G_x^2  + L_x D_x \sqrt{d} \Delta. \label{eq:lm1tmp4}
\end{split}
\end{equation}
We further have
\begin{align}
  &  \bbE[\mF(x_k, y_k)] - \bbE[\mF(x\opt, y_k)] \\
  = & \bbE[\mF(x_k, y_k)] -\bbE[\mF(\hat x_k, y_k)] + \bbE[\mF(\hat x_k, y_k)] - \bbE[\mF(x\opt, y_k)] \\
    \leq & \bbE [\| \mF(x_k, y_k) - \mF(\hat x_k, y_k) \|] + \bbE[\mF(\hat x_k, y_k)] - \bbE[\mF(x\opt, y_k)] \\
    \leq & LD_x\sqrt{d} \Delta + \bbE[\mF(\hat x_k, y_k)] - \bbE[\mF(x\opt, y_k)] \label{eq:lm1tmp3}
\end{align}
where \eqref{eq:lm1tmp3} can be proved by applying qunatization error, function Lipschitz, and diameter bound. Combine \eqref{eq:lm1tmp4} and \eqref{eq:lm1tmp3} to get  \eqref{eq:lm1} in the lemma.
\end{proof}

\begin{lemma} \label{lm2}
Suppose $\mF(x,y)$ is concave in $y$ and has Lipschitz gradients, i.e., $\| \nabla_y \mF(x_1, y) - \nabla_y \mF(x_2, y)\| \leq L_y \| x_1 - x_2 \|  $;  and bounded variance, $\bbE[\| g_x (x,y)\|^2] \leq G_x^2$, $\bbE[\| g_y (x,y)\|^2] \leq G_y^2$; and $\bbE[\| y_k - y\opt \|^2] \leq D_y^2$, we have 
\begin{equation} \label{eq:lm2}
\begin{split}
 \bbE[\mF(x_k, y\opt)] - \bbE[\mF(x_k, y_k)]  \leq & \frac{1}{2\beta_k}\left(\bbE[\| y_k - y\opt \|^2] -\bbE[\|  y\kp - y\opt \|^2] \right) \\
& ~  + \frac{\beta_k}{2} G_y^2 +  2L_y D_y \sqrt{d} \Delta  + \frac{L_y \alpha_k}{2} \, (G_x^2 + D_y^2).
\end{split}
\end{equation}
\end{lemma}

\begin{proof}
From the gradient ascent step, we have 
\begin{align}
\|  y\kp - y\opt \| ^2 &  =  \| y_k + \beta_k \, g_y(x \kp, y_k) - y\opt\|^2 \\
& = \| y_k - y\opt\|^2  + 2 \beta_k \, \inprod{ g_y(x \kp, y_k), \, y_k - y\opt} + \beta_k^2 \, \|  g_y(x \kp, y_k)\|^2.
\end{align} 
Take expectation on both sides of the equation, substitute $\bbE[g_y (x,y)] = \nabla_y \mF(x,y),$ and apply $\bbE[\| g_y^2 (x,y)\|] \leq G_y^2$ to get
\begin{align}
\bbE[\|  y\kp - y\opt \| ^2] \leq \bbE[\| y_k - y\opt\|^2] + 2 \beta_k \, \bbE[\inprod{ \nabla_y \mF(x \kp, y_k), \, y_k - y\opt}] + \beta_k^2 \, G_y^2 . \label{eq:lm2tmp1}
\end{align} 
Reorganize \eqref{eq:lm2tmp1} to get 
\begin{align}
\bbE[\|  y\kp - y\opt\|^2] - \bbE[\| y_k - y\opt \| ^2] - \beta_k^2 \, G_y^2 \leq 2 \beta_k \, \bbE[\inprod{ \nabla_y \mF(x \kp, y_k), \, y_k - y\opt}] . \label{eq:lm2tmp2}
\end{align} 
The right hand side of \eqref{eq:lm2tmp2} can be represented as 
\begin{align}
& \bbE[\inprod{ \nabla_y \mF(x \kp, y_k), \, y_k - y\opt}] \\
= & \bbE[\inprod{ \nabla_y \mF(x \kp, y_k) - \nabla_y \mF(x_k, y_k) + \nabla_y \mF(x_k, y_k), \, y_k - y\opt}] \\
=  & \bbE[\inprod{ \nabla_y \mF(x \kp, y_k) - \nabla_y \mF(x_k, y_k) , \, y_k - y\opt}] + \bbE[\inprod{ \nabla_y \mF(x_k, y_k), \, y_k - y\opt}], \label{eq:lm2tmp8}
\end{align}
where 
\begin{align}
& \bbE[\inprod{ \nabla_y \mF(x \kp, y_k) - \nabla_y \mF(x_k, y_k) , \, y_k - y\opt}] \\
\leq & \bbE[ \| \nabla_y \mF(x \kp, y_k) - \nabla_y \mF(x_k, y_k) \| \, \| y_k - y\opt \|] \\
\leq &  \bbE[ L_y \, \| x \kp - x_k \| \, \| y_k - y\opt \|]  \label{eq:lm2tmp3} \\
= &  L_y  \bbE[ \, \| \mQ(\hat x\kp) - \mQ(\hat x_k) \| \, \| y_k - y\opt \|]\\
=& L_y \bbE[  \, \|( \mQ(\hat x\kp) -\hat x\kp) - (\mQ(\hat x_k)-\hat x_k) + (\hat x\kp - \hat x_k) \| \, \| y_k - y\opt \|]\\
\leq & L_y \bbE[  \, (\| \mQ(\hat x\kp   - \hat x\kp \| + \| \mQ(\hat x_k)-\hat x_k \| + \|\alpha_k g_x(x_k, y_k) \| ) \, \| y_k - y\opt \|]\\
\leq & L_y \bbE[  \, (\sqrt{d} \Delta + \sqrt{d} \Delta + \|\alpha_k g_x(x_k, y_k) \| ) \, \| y_k - y\opt \|] \label{eq:lm2tmp4}\\
\leq & 2L_y \sqrt{d} \Delta \, \bbE[\| y_k - y\opt \|] + \bbE [\|\alpha_k g_x(x_k, y_k) \|  \, \| y_k - y\opt \|]\\
\leq & 2L_y \sqrt{d} \Delta \, \bbE[\| y_k - y\opt \|] + \frac{L_y  \alpha_k}{2} \, \bbE[ \, \|  g_x(x_k, y_k) \|^2 +  \| y_k - y\opt \|^2] \\
\leq & 2L_y \sqrt{d} \Delta D_y + \frac{L_y  \alpha_k}{2} \, (G_x^2 + D_y^2)  \label{eq:lm2tmp6}. 
\end{align}
Lipschitz smoothness is used for \eqref{eq:lm2tmp3}; quanitzation error bound is used in \eqref{eq:lm2tmp4}, which is independent of the stochasticity. From the convexity of quadratic function, we have $\bbE[\| y_k - y\opt \|] \leq \sqrt{\bbE[\| y_k - y\opt \|^2]} \leq D_y $ to get \eqref{eq:lm2tmp6}
Since $\mF(x,y)$ is concave in $y$, we have 
\begin{equation}
\inprod{ \nabla_y \mF(x_k, y_k), \, y_k - y\opt} \leq \mF(x_k, y_k) - \mF(x_k, y\opt).\label{eq:lm2tmp7}
\end{equation}

Combine equations (\ref{eq:lm2tmp2}, \ref{eq:lm2tmp8}, \ref{eq:lm2tmp6} to \ref{eq:lm2tmp7})
\begin{equation}
\begin{split}
& \frac{1}{2\beta_k} \, \left( \bbE[\|  y\kp - y\opt\|^2] - \bbE[\| y_k - y\opt \| ^2] \right) - \frac{ \beta_k}{2} G_y^2 \\
 & \leq  2L_y \sqrt{d} \Delta D_y + \frac{L_y \alpha_k}{2} \, (G_x^2 + D_y^2) + \bbE[\mF(x_k, y_k) - \mF(x_k, y\opt)].
 \end{split}\label{eq:lm2tmp9}
\end{equation}
Rearrange the order of \eqref{eq:lm2tmp9} to achieve \eqref{eq:lm2}.
\end{proof}

\vspace{2em}
\textbf{Proof for Theorem~\ref{thm:quantization_main}}

\begin{proof}
Combining \eqref{eq:lm1} and \eqref{eq:lm2} in the Lemmas, the primal-dual gap $P(x_k, y_k) = \mF(x_k, y\opt) - \mF(x\opt, y_k)$ satisfies,
\begin{equation} \label{eq:thm1p1}
\begin{split}
\bbE[P(x_k, y_k)] \leq & \frac{1}{2\alpha_k}\left(\bbE[\| \hat x_k - x\opt \|^2] -\bbE[\| \hat x\kp - x\opt \|^2] \right) + \frac{\alpha_k}{2} G_x^2  + (L_x+L) D_x \sqrt{d} \Delta \\
& +  \frac{1}{2\beta_k}\left(\bbE[\| y_k - y\opt \|^2] -\bbE[\|  y\kp - y\opt \|^2] \right) + \frac{\beta_k}{2} G_y^2 +  2L_y \sqrt{d}D_y \Delta \\
& + \frac{L_y \alpha_k}{2} \, (G_x^2 + D_y^2).
\end{split}
\end{equation}

Accumulate \eqref{eq:thm1p1} from $k=1,\ldots, l$ to obtain
\begin{equation} 
\begin{split}
& \sum_{k=1}^l \bbE[P(x_k, y_k)] \leq \\ 
&  \frac{1}{2\alpha_1} \bbE[\| x^1 - x\opt \|^2]  + \sum_{k=2}^{l} (\frac{1}{2\alpha_k} - \frac{1}{2\alpha_{k-1}}) \bbE[\| x_k - x\opt \|^2] + \sum_{k=1}^{l} \frac{\alpha_k}{2}( G_x^2 + L_y G_x^2 + L_y D_y^2) \\
& + \frac{1}{2\beta_1} \bbE[\| y^1 - y\opt \|^2]  + \sum_{k=2}^{l} (\frac{1}{2\beta_k} - \frac{1}{2\beta_{k-1}}) \bbE[\| y_k - y\opt \|^2] + \sum_{k=1}^{l}  \frac{\beta_k}{2} G_y^2 \\
& + l (L_x D_x+LD_x + 2L_yD_y)\sqrt{d}.
\end{split}
\end{equation}
Assume $\bbE[|| x_k - x\opt\| ^2] \leq D_u^2, \, \bbE[|| y_k - y\opt\| ^2] \leq D_y^2$ are bounded, we have 
\begin{equation} 
\begin{split}
\sum_{k=1}^l \bbE[P(x_k, y_k)] \leq &  \frac{1}{2\alpha_1} D_x^2  + \sum_{k=2}^{l} (\frac{1}{2\alpha_k} - \frac{1}{2\alpha_{k-1}}) D_u^2 + \sum_{k=1}^{l} \frac{\alpha_k}{2} ( Gx^2 + L_y G_x^2 + L_y D_y^2) \\
& + \frac{1}{2\beta_1} D_y^2  + \sum_{k=2}^{l} (\frac{1}{2\beta_k} - \frac{1}{2\beta_{k-1}}) D_y^2 + \sum_{k=1}^{l}  \frac{\beta_k}{2} G_y^2 \\
& + l (L_x D_x+LD_x + 2L_yD_y)\sqrt{d}.
\end{split}
\end{equation}

Since $\alpha_k, \beta_k$ are decreasing and $\sum_{k=1}^l \alpha_k \leq C_\alpha \sqrt{l+1}, \, \sum_{k=1}^l \beta_k \leq C_\beta \sqrt{l+1} $, we have 
\begin{equation} 
\begin{split}
\sum_{k=1}^l \bbE[P(x_k, y_k)] \leq \frac{\sqrt{l}}{2} \left(\frac{D_x^2}{C_\alpha} + \frac{D_y^2}{C_\beta}\right) 
& + \frac{\sqrt{l+1}}{2} \left(C_\alpha G_x^2 + C_\alpha L_y G_x^2 + C_\alpha L_y D_y^2 + C_\beta G_y^2\right)  \\
    & +  l(L_x D_x+LD_x + 2L_yD_y)\sqrt{d} \Delta \label{eq:thmp2}
\end{split}
\end{equation}

For $\bar x^l = \frac{1}{l} \sum_{k=1}^l x_k, \, \bar y^l = \frac{1}{l} \sum_{k=1}^l y_k $, because $\mF(x, y)$ is convex-concave, we have
\begin{align}
\bbE[P(\bar x^l, \bar y^l)] & = \bbE[\mF (\bar x^l, y\opt) - \mF (y\opt, \bar y^l)] \\
& \leq \bbE[ \frac{1}{l} \sum_{k=1}^{l}( \mF (x_k, y\opt) - \mF (x\opt, y_k)) ] \\
& = \frac{1}{l} \sum_{k=1}^{l}\bbE[\mF (x_k, y\opt) - \mF (x\opt, y_k)] \\
 & =  \frac{1}{l} \sum_{k=1}^{l} \bbE[P(x_k, y_k)]. \label{eq:thmp3}
\end{align}
Combine \eqref{eq:thmp2} and \eqref{eq:thmp3} to prove
\begin{equation}
\begin{split}
\bbE[P(\bar x^l, \bar y^l)]  \leq \frac{1}{2\sqrt{l}} \left(\frac{D_x^2}{C_\alpha} + \frac{D_y^2}{C_\beta}\right) & + \frac{\sqrt{l+1}}{2l} \left(C_\alpha G_x^2 + C_\alpha L_y G_x^2 + C_\alpha L_y D_y^2 + C_\beta G_y^2\right)  \\
    & +  (L_x D_x+LD_x + 2L_yD_y)\sqrt{d} \Delta. 
\end{split}
\end{equation}
\end{proof}

\section{Proofs for Data-free Pruning}
The proof is a simplification of~\cite{yang2020global} assuming a gradient oracle, i.e., $\nabla \mF(x,y)$ can be obtained at any $(x,y)$.

\subsection{Key Lemmas}
The following lemmas will be used in the main proofs.
\begin{lemma}[PL indicates EB and QG\cite{karimi2016linear}]\label{lemma:pl_eb_qg} Any $l$-smooth function $f(\cdot)$ satisfying PL with constant $\mu$ also satisfies Error Bound (EB) condition with $\mu$, i.e., 
\begin{equation*}
    \lVert \nabla f(x) \rVert \ge \mu \lVert x\opt - x \rVert, \forall x,
\end{equation*}
where $x\opt$ is the projection of $x$ onto the optimal set. 

Such $f(\cdot)$ also satisfies Quadratic Growth (QG) condition with $\mu$, i.e.,
\begin{equation*}
    f(x) - f\opt \ge \frac{\mu}{2} \lVert x\opt - x \rVert^2, \forall x.
\end{equation*}
\end{lemma}

It is easy to derive from the EB condition that $l\ge \mu$.
\begin{lemma}[Smoothness and gradient of $h$~\cite{nouiehed2019solving}]\label{lemma:smooth_grad_h}
In the original minimax problem, if $-\mF(x,\cdot)$ satisfies PL condition with constant $\mu_2$ for any $x$, and $\mF$ is $L$-smooth (Assumption 1), then the function $h(x):=\max_{y} \mF(x,y)$ is $L_h$-smooth with $L_h=L+\frac{L^2}{2\mu_2}$, and $\nabla h(x)=\nabla_x \mF(x,y\opt(x))$ for any $y\opt(x)\in \arg\max_{y}\mF(x,y)$.
\end{lemma}
Also, $h(x)$ satisfies PL condition.
\begin{lemma}[$h$ is $\mu_1$-PL~\cite{yang2020global}]\label{lemma:h_pl}
If $\mF(x,y)$ satisfies Assumption 1 and Assumption 3, then function $h(x):=\max_y \mF(x,y)$ satisfies the PL condition with $\mu_1$.
\end{lemma}

\subsection{Main Proofs}
We first prove a contraction theorem for each iteration in the noiseless setting. 

\begin{theorem}[Contraction of Potential Function]\label{thm:contraction}
Assume Assumptions 1,2,3 hold for $\mF(x,y)$. If we run one iteration of updates in Eq.~\ref{eq:pruning_minimax} with $\alpha_k=\alpha\le 1/L_h$ ($L_h=L+\frac{L^2}{2\mu_2}$ as specified in Lemma~\ref{lemma:smooth_grad_h}) and $\beta_k=\beta\le 1/L$, then
\begin{equation}
    a_{k+1}+\lambda b_{k+1}\le \max\{\gamma_1, \gamma_2\}(a_k+\lambda b_k),
\end{equation}
where
\begin{equation}
    \begin{split}
    \gamma_1 &= 1 - \mu_1 \alpha - \lambda\mu_1(1-\mu_2\beta)\left[\alpha - \left(2\alpha+\alpha^2L\right)\left(1+\frac{1}{\epsilon}\right) \right],\\
    \gamma_2 &= 1 - \mu_2\beta + \frac{\alpha L^2}{\lambda\mu_2} +(1-\mu_2\beta)\frac{L^2}{\mu_2}\left[\left(2\alpha+\alpha^2L\right)\left(1+\epsilon\right)+\alpha\right],
\end{split}
\end{equation}
and $\lambda>0,\epsilon>0$ are constants satisfying 
\begin{equation*}
    \frac{\alpha}{2} + \lambda(1-\mu_2\beta)\left[\frac{\alpha}{2} - \left(\alpha+\frac{\alpha^2L}{2}\right)\left(1+\frac{1}{\epsilon}\right) \right] \ge 0.
\end{equation*}
\end{theorem}

\begin{proof}
We look at $a_{k+1}$ and $b_{k+1}$ separately to derive bounds for the potential function.
Since $h(x)$ is $L_h$-smooth by Lemma~\ref{lemma:smooth_grad_h}, we have 
\begin{equation}\label{eq:a_contract}
\begin{split}
    a_{k+1}=h(x_{k+1})-h\opt  &\le h(x_k) - h\opt + \langle \nabla h(x_k), x_{k+1}-x_k \rangle + \frac{L_h}{2}\lVert x_{k+1} - x_k \rVert^2\\
    &= a_k - \alpha \langle \nabla h(x_k), \nabla_{x}\mF(x_k,y_k) \rangle + \frac{L_h \alpha^2}{2} \lVert \mF(x_k,y_k) \rVert^2 \\
    &\le a_k + \frac{\alpha}{2}\lVert \nabla_x \mF(x_k,y_k) - \nabla h(x_k) \rVert^2 - \frac{\alpha}{2}\lVert \nabla h(x_k) \rVert^2,
\end{split}
\end{equation}
where the second inequality uses the assumption that $\alpha \le 1/L_h$.

The values of $\lVert \nabla_x \mF(x_k,y_k) - \nabla h(x_k) \rVert^2 $ and $\lVert \nabla h(x_k) \rVert^2$ can be bounded by $a_k, b_k$. 
With Lemma~\ref{lemma:smooth_grad_h} and Assumption 1, we have 
\begin{equation*}
    \lVert \nabla_x \mF(x_k,y_k) - \nabla h(x_k) \rVert^2 \le  \lVert \nabla_x \mF(x_k,y_k) - \nabla_x \mF(x_k, y\opt(x_k)) \rVert^2 \le L^2 \lVert y\opt(x_k) - y_k \rVert^2,
\end{equation*}
for $\forall y\opt(x_k)\in \arg\max_y \mF(x_k,y)$. 
Because $-\mF(x_k, y)$ is $\mu_2$-PL in $y$, it has Quadratic Growth as defined in Lemma~\ref{lemma:pl_eb_qg}, so
\begin{equation}\label{eq:diff_bound}
    \lVert \nabla_x \mF(x_k,y_k) - \nabla h(x_k) \rVert^2 \le  L^2 \lVert y\opt(x_k) - y_k \rVert^2 \le \frac{2L^2}{\mu_2} [h(x_k)-\mF(x_k,y_k)]=\frac{2L^2}{\mu_2} b_k.
\end{equation}
For $\lVert \nabla h(x_k) \rVert^2$, we know $g(x)$ is $\mu_1$-PL from Lemma~\ref{lemma:h_pl}, so
\begin{equation}\label{eq:hx_bound}
    \lVert \nabla h(x_k) \rVert^2 \ge 2\mu_1[h(x_k)-h\opt]=2\mu_1 a_k.
\end{equation}

For $b_{k+1}$, we first prove it is a contraction with respect to $y_k$.
Specifically,
\begin{equation}\label{eq:bk0}
    \begin{split}
        b_{k+1}&=h(x_{k+1}) - \mF(x_{k+1}, y_{k+1}) \\
              &\le h(x_{k+1}) - \mF(x_{k+1},y_{k}) - \langle \nabla_y \mF(x_{k+1},y_k),y_{k+1}-y_k \rangle + \frac{L}{2}\lVert y_{k+1}-y_k \rVert^2 \\
              & = h(x_{k+1}) - \mF(x_{k+1},y_{k}) +(\frac{L\beta^2}{2}-\beta)\lVert \nabla_y \mF(x_{k+1},y_k) \rVert^2 \\
              &\le h(x_{k+1}) - \mF(x_{k+1},y_{k}) - \mu_2\beta  \left[h(x_{k+1}) - \mF(x_{k+1},y_{k})\right]\\
              & = (1- \mu_2\beta)\left[h(x_{k+1}) - \mF(x_{k+1},y_{k})\right],
    \end{split}
\end{equation}
where the first inequality comes from the assumption that $\mF(x,y)$ is $L$-smooth in $y$, and the second inequality uses the assumptions that $\beta \le 1/L$ and $\mF(x,y)$ is $\mu_2$-PL in $y$. 
Further, observe that 
\begin{equation}\label{eq:bk1}
    h(x_{k+1}) - \mF(x_{k+1},y_{k}) = b_k + \mF(x_k,y_k) - \mF(x_{k+1},y_{k}) + h(x_{k+1}) - h(x_k). 
\end{equation}
Because $\mF(x,y)$ is $L$-smooth by Assumption 1, we have 
\begin{equation}\label{eq:bk2}
\begin{split}
    \mF(x_k,y_k) - \mF(x_{k+1},y_{k}) 
        & \le -\langle \nabla_x \mF(x_k, y_k), x_{k+1}-x_k \rangle + \frac{L}{2}\lVert x_{k+1}-x_k \rVert^2\\
        & = (\alpha + \frac{\alpha^2 L}{2})\lVert \nabla_x \mF(x_k,y_k) \rVert^2 \\
        & \le (\alpha + \frac{\alpha^2 L}{2})[(1+\epsilon) \lVert \nabla_x \mF(x_k,y_k) - \nabla h(x_k) \rVert^2 \\
        &\qquad\qquad\quad + (1+\frac{1}{\epsilon})\lVert \nabla h(x_k) \rVert^2],
\end{split}
\end{equation}
where the second inequality holds according to Young's inequality for any $\epsilon>0$.
From~\ref{eq:a_contract}, we know that 
\begin{equation}\label{eq:bk3}
    h(x_{k+1})-h(x_k) = a_{k+1}-a_k\le \frac{\alpha}{2}\lVert \nabla_x \mF(x_k,y_k) - \nabla h(x_k) \rVert^2 - \frac{\alpha}{2}\lVert \nabla h(x_k) \rVert^2.
\end{equation}
Combining Eq.~\ref{eq:bk0},~\ref{eq:bk1},~\ref{eq:bk2} and~\ref{eq:bk3} together, 
\begin{equation*}
\begin{split}
    b_{k+1}\le & (1-\mu_2\beta)\left\{b_k+\left[\left(\alpha+\frac{\alpha^2L}{2}\right)\left(1+\epsilon\right)+\frac{\alpha}{2}\right]\lVert \nabla_x \mF(x_k,y_k) - \nabla h(x_k) \rVert^2 \right. \\
          &\qquad\qquad \left. - \left[\frac{\alpha}{2} - \left(\alpha+\frac{\alpha^2L}{2}\right)\left(1+\frac{1}{\epsilon}\right)\right] \lVert \nabla h(x_k) \rVert^2 \right\}.
\end{split}
\end{equation*}
Together with Eq.~\ref{eq:diff_bound} and~\ref{eq:hx_bound}, we know that 
\begin{equation}
\begin{split}
    a_{k+1}+\lambda b_{k+1}\le & a_k + \lambda (1-\mu_2\beta) b_k \\
    + \left\{ \frac{\alpha}{2} \right. &\left. +\lambda(1-\mu_2\beta)\left[\left(\alpha+\frac{\alpha^2L}{2}\right)\left(1+\epsilon\right)+\frac{\alpha}{2}\right] \right\}\lVert \nabla_x \mF(x_k,y_k) - \nabla h(x_k) \rVert^2  \\
    - \left\{ \frac{\alpha}{2} \right. & \left.+ \lambda(1-\mu_2\beta)\left[\frac{\alpha}{2} - \left(\alpha+\frac{\alpha^2L}{2}\right)\left(1+\frac{1}{\epsilon}\right) \right]\right\}\lVert \nabla h(x_k) \rVert^2 \\
    \le &  \left\{1 - \mu_1 \alpha - \lambda\mu_1(1-\mu_2\beta)\left[\alpha - \left(2\alpha+\alpha^2L\right)\left(1+\frac{1}{\epsilon}\right) \right]\right\}a_k\\
     & + \lambda\left\{ 1 - \mu_2\beta + \frac{\alpha L^2}{\lambda\mu_2} +(1-\mu_2\beta)\frac{L^2}{\mu_2}\left[\left(2\alpha+\alpha^2L\right)\left(1+\epsilon\right)+\alpha\right] \right\}b_k,\\
     \le & \max \{\gamma_1, \gamma_2\} (a_k+\lambda b_{k})
\end{split}
\end{equation}
where we have defined 
\begin{equation*}
\begin{split}
    \gamma_1 &= 1 - \mu_1 \alpha - \lambda\mu_1(1-\mu_2\beta)\left[\alpha - \left(2\alpha+\alpha^2L\right)\left(1+\frac{1}{\epsilon}\right) \right],\\
    \gamma_2 &= 1 - \mu_2\beta + \frac{\alpha L^2}{\lambda\mu_2} +(1-\mu_2\beta)\frac{L^2}{\mu_2}\left[\left(2\alpha+\alpha^2L\right)\left(1+\epsilon\right)+\alpha\right],
\end{split}
\end{equation*}
and the second inequality requires 
\begin{equation*}
    \frac{\alpha}{2} + \lambda(1-\mu_2\beta)\left[\frac{\alpha}{2} - \left(\alpha+\frac{\alpha^2L}{2}\right)\left(1+\frac{1}{\epsilon}\right) \right] \ge 0.
\end{equation*}
In addition, the contraction requires both $\gamma_1<1$ and $\gamma_2 < 1$.  
\end{proof}

With the results from Theorem~\ref{thm:contraction}, we prove the linear convergence to stationary points for a class of nonconvex-nonconcave objective functions under proper choice of learning rates.

\vspace{2em}
\textbf{Proof of Theorem~\ref{thm:grad_convergence}}

\begin{proof}
We first prove that with $\alpha=\frac{\mu_2^2}{18L^3}$ and $\beta=\frac{1}{L}$, the potential function converges as
\begin{equation}\label{eq:pot_converge}
    P_k \le \left(1-\frac{\mu_1\mu_2^2}{36L^3}\right)^{k}P_0.
\end{equation}
Recall that Theorem~\ref{thm:contraction} requires $\alpha\le \frac{1}{L_h}\le \frac{2}{3L}$ (using the corollary that $L\ge \mu_2$ from the Error Bound of Lemma~\ref{lemma:pl_eb_qg}), and $\beta \le \frac{1}{L}$. 
Let $\lambda=\frac{1}{10}$ and $\epsilon=1$ in Theorem~\ref{thm:contraction}. 
We have
\begin{equation}
\begin{split}
    \gamma_1 &= 1 - \mu_1 \alpha\left\{ 1 + \lambda(1-\mu_2\beta)\left[1 - \left(2+\alpha L\right)\left(1+\frac{1}{\epsilon}\right) \right]\right\}\\
       &\le 1 - \mu_1\alpha\left[ 1 + \frac{1}{10}(1-\mu_2\beta)\left(1-6\right) \right]\\
       &\le 1-\frac{1}{2}\mu_1\alpha,
\end{split}
\end{equation}
where the first inequality plugs in the values of $\lambda, \epsilon$ and uses the fact that $\alpha\le \frac{2}{3L}\le \frac{1}{L}$.
With an additional assumption that $\frac{\mu_2^2\beta}{\alpha L^2}\ge \frac{52}{3}$ (which is satisfied when $\alpha=\frac{\mu_2^2}{18L^3}$ and $\beta=\frac{1}{L}$), 
\begin{equation}
    \begin{split}
        \gamma_2 &= 1 - \frac{\alpha L^2}{\mu_2}\left\{ \frac{\mu_2^2\beta}{\alpha L^2} - \frac{1}{\lambda} - (1-\mu_2\beta)\left[\left(2+\alpha L\right)\left(1+\epsilon\right)+1\right] \right\} \\
        & \le 1 -  \frac{\alpha L^2}{\mu_2} \left[ \frac{\mu_2^2\beta}{\alpha L^2} -10-\frac{19}{3}(1-\mu_2\beta) \right]\\
        & \le 1-\frac{\alpha L^2}{\mu_2},
    \end{split}
\end{equation}
where the first inequality plugs in the values of $\lambda$ and $\epsilon$, and uses the fact that $\alpha\le \frac{2}{3L}$.
Again, using the corollary from EB of Lemma~\ref{lemma:pl_eb_qg}, we know that $\frac{\mu_1\mu_2}{2L^2}<1$, therefore $\frac{1}{2}\mu_1\alpha < \frac{\alpha L^2}{\mu_2}$ and $\gamma_1 > \gamma_2$. 
Plug in the value of $\alpha$ and $\gamma_1$, we reach the conclusion of Eq.~\ref{eq:pot_converge}.

Finally, we prove the convergence rate of $\lVert \nabla_x \mF(x_k,y_k) \rVert^2 + \lVert \nabla_y \mF(x_k,y_k) \rVert^2 $ by upper bounding it with the potential function, which is similar to the proof of~\cite{yang2020global}. 

First,
\begin{equation}
\begin{split}
    \lVert \nabla_x \mF(x_k,y_k) \rVert^2 &\le \lVert \nabla h(x_k) \rVert^2 + \lVert \nabla_x \mF(x_k,y_k) - \nabla g(x_k) \rVert^2 \\ 
    & = \lVert \nabla h(x_k) - \nabla h(x\opt) \rVert^2 + \lVert \nabla_x \mF(x_k,y_k) - \nabla g(x_k) \rVert^2\\
    & \le L_h^2 \lVert x_k-x\opt \rVert^2 + L^2\lVert y\opt(x_k)-y_k \rVert^2 \\
    & \le \frac{2L_h^2}{\mu_1}a_k + \frac{2L^2}{\mu_2} b_k,
\end{split}
\end{equation}
where the the second inequality are based on Lemma~\ref{lemma:smooth_grad_h}, and the last inequality is based on Lemma~\ref{lemma:h_pl} and the Quadratic Growth property in~\ref{lemma:pl_eb_qg}. 

Second,
\begin{equation}
\begin{split}
    \lVert \nabla_y \mF(x_k,y_k) \rVert^2 &\le \lVert \nabla_y \mF(x_k, y_k) - \nabla_y \mF(x_k,y\opt(x_k)) \rVert^2 \\
    & \le L^2 \lVert y_k-y\opt(x_k) \rVert^2 \\
    & \le \frac{2L^2}{\mu_2}b_k,
\end{split}
\end{equation}
where the last inequality comes from the Quadratic Growth property for $\mF(x_k,\cdot)$.
As a result, 
\begin{equation}
\begin{split}
    \lVert \nabla_x \mF(x_k,y_k) \rVert^2 + \lVert \nabla_y \mF(x_k,y_k) \rVert^2 &\le \frac{2L_h^2}{\mu_1}a_k+\frac{4L^2}{\mu_2}b_k \\
    & \le \max\{\frac{2L_h^2}{\mu_1}, \frac{40L^2}{\mu_2}\}(a_k+\frac{1}{10}b_k) \\
    & \le \max\{\frac{2L_h^2}{\mu_1}, \frac{40L^2}{\mu_2}\} \left(1-\frac{\mu_1\mu_2^2}{36L^3}\right)^{k} P_0 .
\end{split}
\end{equation}
\end{proof}

\end{document}